\documentclass[10pt,twocolumn,letterpaper]{article}
\pdfoutput=1

\usepackage{algorithm}
\usepackage{algorithmic}
\usepackage{amsmath}
\usepackage{amssymb}
\usepackage{amsthm}
\usepackage[backend=biber,sortcites]{biblatex}
\usepackage{booktabs}
\usepackage[font=small,labelfont=bf,textfont=it]{caption}
\usepackage[margin=2.5cm]{geometry}
\usepackage{graphicx}
\usepackage[breaklinks=true,letterpaper=true,bookmarks=false]{hyperref}
\usepackage{mathtools}
\usepackage{parskip}
\usepackage{subfigure}
\usepackage{svg}
\usepackage{siunitx}
\usepackage{thmtools}
\usepackage{thm-restate}

\usepackage[capitalize,noabbrev]{cleveref}

\title{Learning Stationary Markov Processes with Contrastive Adjustment}

\author{
  Ludvig Bergenstråhle\textsuperscript{1},
  Jens Lagergren\textsuperscript{2}, and
  Joakim Lundeberg\textsuperscript{1,}%
  \thanks{Correspondence to: \href{mailto:joakim.lundeberg@scilifelab.se}{joakim.lundeberg@scilifelab.se}}
  \\
  {\small
    \textsuperscript{1}Division of Gene Technology,
    \textsuperscript{2}Division of Computational Science and Technology
  }
  \\
  {\small KTH Royal Institute of Technology}
}

\begin{document}

\maketitle

\begin{abstract}
  We introduce a new optimization algorithm, termed \emph{contrastive adjustment}, for learning Markov transition kernels whose stationary distribution matches the data distribution.
  Contrastive adjustment is not restricted to a particular family of transition distributions and can be used to model data in both continuous and discrete state spaces.
  Inspired by recent work on noise-annealed sampling, we propose a particular transition operator, the \emph{noise kernel}, that can trade mixing speed for sample fidelity.
  We show that contrastive adjustment is highly valuable in human-computer design processes, as the stationarity of the learned Markov chain enables local exploration of the data manifold and makes it possible to iteratively refine outputs by human feedback.
  We compare the performance of noise kernels trained with contrastive adjustment to current state-of-the-art generative models and demonstrate promising results on a variety of image synthesis tasks.
\end{abstract}

\section{Introduction}

\begin{figure}[t]
  \centering
  \includegraphics[width=\linewidth]{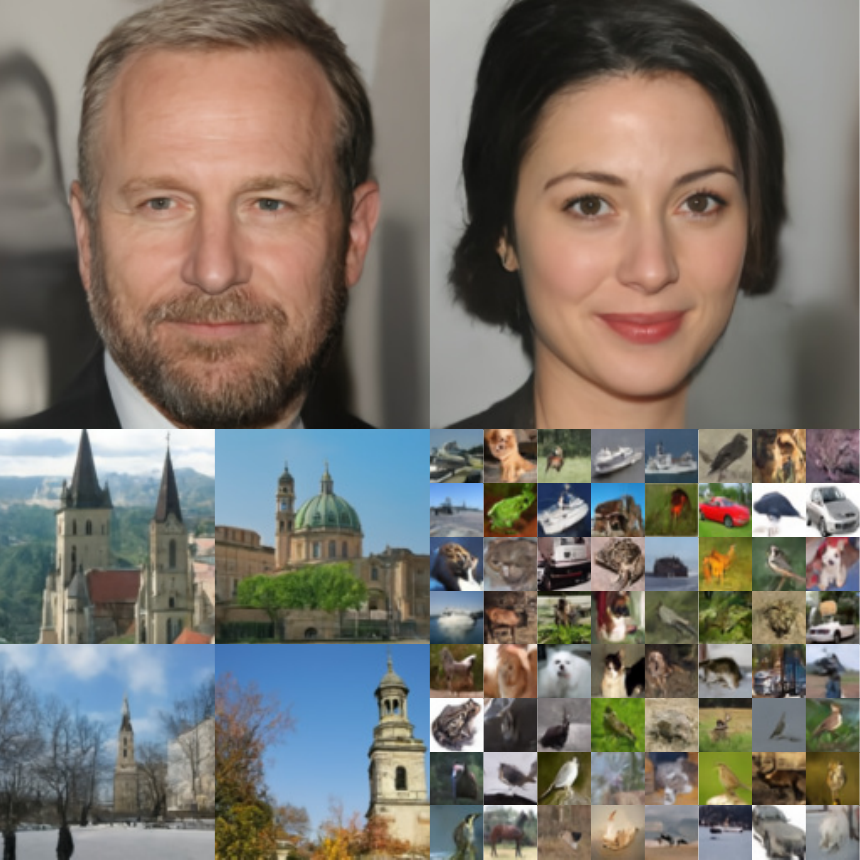}%
  \caption{
    Unconditional samples from noise kernels trained with contrastive adjustment.
    Top: CelebA-HQ ($256 \times 256$).
    Bottom left: LSUN Church ($128 \times 128$).
    Bottom right: CIFAR-10 ($32 \times 32$).
  }%
  \label{fig:ca:samples}
\end{figure}

Generative models have emerged as a powerful tool with applications spanning a wide range of subject areas, from image, text, and audio synthesis in the creative arts to data analysis and molecule design in the sciences.
The literature on generative models encompasses many different methods, including those based on variational autoencoders~\cite{kingma2013auto}, autoregression~\cite{oord2016conditional}, normalizing flows~\cite{rezende2015variational}, adversarial optimization~\cite{goodfellow2014generative}, and Markov chain Monte Carlo (MCMC)~\cite{song2019generative}.
While often very general and applicable to many domains, different methods have distinct strengths and limitations.
For example, adversarial methods can typically generate high-quality samples but are difficult to train on diverse datasets, while autoregressive models are easier to train but slow to sample from as they require sequential generation.

In this paper, we introduce contrastive adjustment, an optimization algorithm for learning arbitrary Markov transition operators whose stationary distribution matches an unknown data distribution.
The learned transition operators obey detailed balance and can be used to quickly generate variants of existing data, making them exceptionally suited for human-in-the-loop design processes.
Additionally, inspired by recent work~\cite{song2019generative,ho2020denoising} on noise-annealed sampling, we propose noise kernels, a specific transition model that can be learned with contrastive adjustment and allows for efficient de novo synthesis by modeling the data distribution over multiple noise levels.
We find that noise kernels trained with contrastive adjustment produce diverse, high-quality samples on image datasets of different resolutions (\cref{fig:ca:samples}).
To further demonstrate their flexibility, we allow noise kernels to be conditioned on partial data and illustrate how they can be used to inpaint missing image regions without retraining.
We find that contrastive adjustment and noise kernel transition models are straightforward to implement, can be applied to both discrete and continuous data domains, and show promising performance on a variety of image synthesis tasks.

\section{Method}

\begin{figure}[t]
  \centering
  \includegraphics[width=\linewidth]{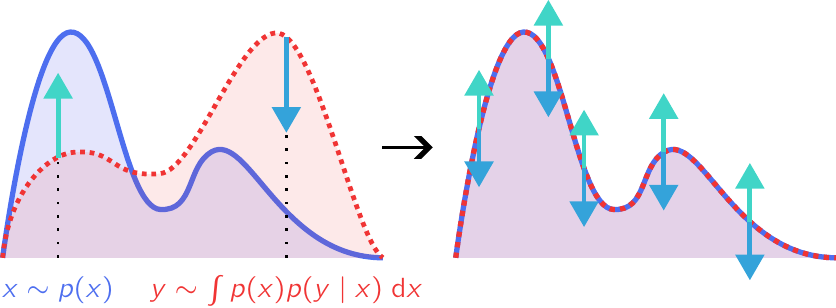}%
  \caption{
    Conceptual illustration of contrastive adjustment.
    Left: The transition kernel is adjusted based on the density difference between the data distribution (blue solid line) and its one-step perturbation induced by the current kernel (red dotted line).
    Right: In steady state, expected upward and downward adjustments cancel out. Assuming symmetric adjustments, this happens when the distributions of perturbed and unperturbed data coincide.
  }%
  \label{fig:ca:high-level}
\end{figure}

The goal of contrastive adjustment is to learn a reversible Markov transition kernel $p(y \mid x) = p(X_{t+1} = y \mid X_{t} = x)$ whose stationary distribution matches an unknown data distribution $p(x)$.
At optimality, the transition distribution adheres to the detailed balance criterion,
\begin{equation}
  p(x) p(y \mid x) = p(y) p(x \mid y).
\end{equation}
Contrastive adjustment learns the transition distribution by iterative adjustments;
when $B_{xy} = p(x)p(y \mid x) - p(y)p(x \mid y) > 0$ for some values $x,y$, we decrease $p(y \mid x)$ and increase $p(x \mid y)$.
Note that since the data distribution $p(x)$ is unknown, we cannot evaluate $B_{xy}$ directly.
The central idea of contrastive adjustment is that directionally accurate updates to the transition distribution can still be estimated by Monte Carlo sampling.
Specifically, we take a sample from the one-step process $(x,y) \sim p(x)p(y \mid x)$ and decrease the probability of the forward transition $p(y \mid x)$ while increasing the probability of the backward transition $p(x \mid y)$ (\cref{alg:ca:high-level}).
Assuming symmetric updates and disregarding normalization constraints, the expected adjustment of $p(y \mid x)$ is negative when $B_{xy} > 0$ and positive when $B_{xy} < 0$.
Incremental updates proceed until $p(x)p(y \mid x) = p(y)p(x \mid y)$ for all $x,y$, whereby downward and upward adjustments cancel out in expectation (\cref{fig:ca:high-level}).

In most practical applications, $p_{\theta}(y \mid x)$ will be parameterized by a neural network with weights $\theta$.
While several different adjustment functions are possible, a natural choice is updating $\theta$ in the direction or anti-direction of the gradient of the log transition probabilities.
Noting that the expected update to $\theta$ from downward adjustments then is
\begin{align}
  \mathbb{E}_{p(x)p(y \mid x)}\left[-\nabla_{\theta} \log p_{\theta}(y \mid x) \right]
  & \nonumber \\
  & \hspace{-4.0cm}
    = -\iint p(x)p_{\theta}(y \mid x) \nabla_{\theta} \log p_{\theta}(y \mid x)  \; \mathrm{d}x \mathrm{d}y
  \\
  & \hspace{-4.0cm}
    = -\iint p(x) \nabla_{\theta} p_{\theta}(y \mid x)  \; \mathrm{d}x \mathrm{d}y
  \\
  & \hspace{-4.0cm}
    = -\nabla_{\theta} \iint p(x) p_{\theta}(y \mid x) \; \mathrm{d}x \mathrm{d}y
    = -\nabla_{\theta} 1 = 0,
\end{align}
we can implement gradient-based contrastive adjustment without performing the downward adjustment step, as described in \cref{alg:ca:gradient}.

\begin{algorithm}[t]
  \caption{High-level description}%
  \label{alg:ca:high-level}
  \begin{algorithmic}
    \STATE{\bfseries Input:} data distribution~$p(x)$, transition distribution~$p(y \mid x)$%
    \REPEAT%
    \STATE~Sample $x \sim p(x)$%
    \STATE~Sample $y \sim p(y \mid x)$%
    \STATE~Decrease $p(y \mid x)$%
    \STATE~Increase $p(x \mid y)$%
    \UNTIL{convergence}%
  \end{algorithmic}
\end{algorithm}

\begin{algorithm}[t]
  \caption{Gradient-based adjustments}%
  \label{alg:ca:gradient}
  \begin{algorithmic}
    \STATE{\bfseries Input:} data distribution~$p(x)$, transition model~$p_{\theta}(y \mid x)$, weights~$\theta$, learning rate~$\eta$%
    \REPEAT%
    \STATE~Sample $x \sim p(x)$%
    \STATE~Sample $y \sim p_{\theta}(y \mid x)$%
    \STATE~$g \gets \nabla_{\theta} \log p_{\theta}(x \mid y)$%
    \STATE~$\theta \gets \theta + \eta g$%
    \UNTIL{convergence}%
  \end{algorithmic}
\end{algorithm}

\subsection{Noise kernel transition models}%
\label{sec:ca:noise-kernel}

The transition model determines the dynamics of the learned Markov chain and is thus a crucial component of contrastive adjustment.
It should be able to capture the modes of the data but also allow for efficient sampling.
One of the key insights of~\textcite{song2019generative} is that it is possible to trade off between these two objectives by modeling the data over different noise levels.
Specifically, at higher noise levels, the data distribution is flatter, allowing the transition kernel to take larger steps without falling off the data manifold.
At lower noise levels, the mixing speed is slower but sample fidelity higher.
Consequently, one can attain both efficient mixing and high sample fidelity by switching between high and low noise regimes.

Motivated by this idea, we propose to learn a specific class of transition distributions, which we will refer to as noise kernels, for modeling noisy representations of the data.
To define a noise kernel, let $p(z)$ be the data distribution and $p(x \mid z)$ a distribution that adds noise to a non-noisy example $z$.
Suppose the marginal density for a transition from $x$ to $y$ at some step in the chain can be modeled by the joint distribution
\begin{align}
  p(z, x, y) = p(z)p(x \mid z)p(y \mid z, x),
\end{align}
where the last term is the conditional transition probability from $x$ to $y$, defined so that
\begin{align}
  p(x \mid z) p(y \mid z, x) = p(y \mid z) p(x \mid z, y).
  \label{eq:noise-kernel:db}
\end{align}
Under this framework, the transition distribution is given by the marginal
\begin{align}
  p(y \mid x) = \int p(z \mid x) p(y \mid z, x) \; \mathrm{d}z,
  \label{eq:noise-kernel:transition}
\end{align}
where $p(z \mid x)$ is a denoising distribution.
Note that the transition distribution~\eqref{eq:noise-kernel:transition} adheres to the detailed balance criterion, since
\begin{align}
  p(x)p(y \mid x)
  & = \int p(z)p(x \mid z)p(y \mid z, x) \; \mathrm{d}z
    \\
  & = \int p(z)p(y \mid z)p(x \mid z, y) \; \mathrm{d}z
    \\
  & = p(y)p(x \mid y),
  \label{eq:noise-kernel:unconditional-db}
\end{align}
where the second equality follows from \cref{eq:noise-kernel:db}.

In general, we cannot solve \cref{eq:noise-kernel:transition} for the denoising distribution $p(z \mid x)$ analytically, but we can replace it with a learnable distribution $r_{\theta}(z \mid x)$ parameterized by weights $\theta$, giving us the approximate transition model
\begin{align}
  p_{\theta}(y \mid x)
  & =
    \int r_{\theta}(z \mid x) p(y \mid z, x) \; \mathrm{d}z
    .
    \label{eq:noise-kernel:model}
\end{align}
By selecting $r_{\theta}$ appropriately, we can ensure the integral in \cref{eq:noise-kernel:model} is tractable and that $p_{\theta}$ is identifiable by the distribution parameters of $r_{\theta}$ (cf. \cref{sec:nk:continuous,sec:nk:categorical}).
The distribution $r_{\theta}$ can then be used as a denoiser to reconstruct samples in observation space.

Since noise kernels are equilibrium models according to \cref{eq:noise-kernel:unconditional-db}, they are learnable by contrastive adjustment.
In this case, the data distribution in \cref{alg:ca:gradient} is the noisy distribution $p(x) = \int p(z)p(x \mid z) \; \mathrm{d}z$.
While it should be noted that there could be multiple transition models $p_{\theta}$ consistent with detailed balance, we did not find this to be a problem in practice.
As an alternative to contrastive adjustment, noise kernels can also be learned by optimizing $r_{\theta}$ with a reconstruction loss or by minimizing the Kullback-Leibler divergence between \cref{eq:noise-kernel:transition} and \cref{eq:noise-kernel:model}.
However, we generally found contrastive adjustment to produce better results.

\subsection{Non-equilibrium transitions}%
\label{sec:ca:non-equilibrium}

As discussed in \cref{sec:ca:noise-kernel}, efficient mixing and high sample fidelity can be achieved by switching between high and low noise regimes.
While contrastive adjustment learns a reversible Markov chain, it is straightforward to extend \cref{eq:noise-kernel:model} to allow for non-equilibrium transitions over different noise levels at inference time.
Specifically, instead of requiring the conditional transition distribution to obey the detailed balance criterion~\eqref{eq:noise-kernel:db}, we now let it depend on the step $t$ of the chain and require only that
\begin{align}
  \int p^{t}(x \mid z) p^{t}(y \mid z, x) \; \mathrm{d}z
  = p^{t+1}(y \mid z),
  \label{eq:noise-kernel:non-equilibrium-condition}
\end{align}
where $p^{t}(\cdot \mid z)$ is a time-dependent noise distribution.
The transition model is then defined as
\begin{align}
  p^{t}_{\theta}(y \mid x) = \int r^{t}_{\theta}(z \mid x) p^{t}(y \mid z, x) \; \mathrm{d}z,
\end{align}
where we let the denoising model $r^{t}_{\theta}(z \mid x) = r_{\theta}(z \mid x, \beta_{t})$ depend on the noise level $\beta_{t}$ at step $t$.

Training is the same as in \cref{sec:ca:noise-kernel}, except that a separate transition model is learned for each noise level.
In practice, we share the weights $\theta$ of the denoising model $r_{\theta}$ over an infinite number of noise levels $\beta \in (0, 1)$ and, similar to~\cite{ho2020denoising}, condition $r_{\theta}$ on a sinusoidal position embedding of $\beta$.
During training, we sample $\beta$ uniformly from the range $(0, 1)$ for each example.

\subsection{Continuous noise kernels}%
\label{sec:nk:continuous}

We are now ready to define concrete examples of noise kernels for continuous and categorical data.
In the continuous case, we will use an isotropic Gaussian noise distribution:
\begin{align}
  \label{eq:nk:continuous:noise}
  p^{t}(x \mid z)
  & = \mathcal{N}\left(x \mid \alpha_{t} z, \beta_{t} I \right)
  ,
\end{align}
where $\beta_{t}$ and $\alpha_{t}$ determine the noise level at step $t$.
To improve the stability of the kernel, we introduce a dependency on the previous state in the conditional transition distribution:
\begin{equation}
  p^{t}(y \mid z, x) = \mathcal{N}(
  y \mid
  w x + a_{t+1}z
  ,
  b_{t+1}I
  )
  \label{eq:nk:continuous:conditional}
  ,
\end{equation}
where $w$ is a scalar hyperparameter and we have used $I$ to denote the identity matrix.
The scalars $a_{t+1}$ and $b_{t+1}$ can be derived from the conditions~\eqref{eq:noise-kernel:db} and~\eqref{eq:noise-kernel:non-equilibrium-condition} for equilibrium and non-equilibrium transitions, respectively.
In the non-equilibrium case, we can make use of the following proposition:
\begin{restatable}{prop}{gaussianmarginal}%
  \label{prop:gaussian-marginal}
  Consider the joint distribution $p(X,Y) = p(X)p(Y \mid X)$, where $p(X) = N(X \mid \mu, \sigma^{2})$ and $p(Y \mid X) = N(Y \mid aX + b, c^{2})$.
  Then the marginal distribution of $Y$ is given by $p(Y) = N(Y \mid a \mu + b, c^{2} + a^{2}\sigma^{2})$.
\end{restatable}

Since~\eqref{eq:nk:continuous:noise} and~\eqref{eq:nk:continuous:conditional} are diagonal Gaussians, we can apply \cref{prop:gaussian-marginal} dimension-wise with $X = x$, $Y = y$, $\mu = \alpha_{t}z$, $\sigma^{2} = \beta_{t}$, $a = w$, $b = a_{t+1} z$, and $c^{2} = b_{t+1}$ to obtain
\begin{align}
  w \alpha_{t}z + a_{t+1}z = \alpha_{t+1}z
  \Rightarrow
    a_{t+1} = \alpha_{t+1} - w \alpha_{t}
  \label{eq:nk:continuous:alpha}
  \\
  b_{t+1} + w^{2} \beta_{t} = \beta_{t+1}
  \Rightarrow
  b_{t+1} = \beta_{t+1} - w^{2} \beta_{t}
  \label{eq:nk:continuous:beta}
  .
\end{align}
As stated by the following proposition, we do not need to treat the equilibrium case separately:
\begin{restatable}{prop}{gaussiandb}%
  \label{prop:gaussian-db}
  If the noise level is constant over time, so that $\beta_{t} = \beta$ and $\alpha_{t} = \alpha$ for all $t$, then the continuous noise kernel defined by \cref*{eq:nk:continuous:noise,eq:nk:continuous:conditional,eq:nk:continuous:alpha,eq:nk:continuous:beta} satisfies the detailed balance criterion~\eqref{eq:noise-kernel:db}.
\end{restatable}

Proofs of Propositions~\ref{prop:gaussian-marginal} and~\ref{prop:gaussian-db} are provided in \cref{app:proofs}.

To derive the transition model $p^{t}_{\theta}$, we define
\begin{align}
  \label{eq:nk:continuous:denoising}
  r^{t}_{\theta}(z \mid x)
  & = \mathcal{N}\left(z \mid \mu_{\theta}(x, \beta_{t}), \sigma^{2}_{\theta}(x, \beta_{t}) I\right)
    ,
\end{align}
where $\mu_{\theta}$ and $\sigma_{\theta}^{2}$ are deep neural networks parameterized by weights $\theta$.
We can now use \cref{prop:gaussian-marginal} again but with $X = z$, $Y = y$, $\mu = \mu_{\theta}(x, \alpha_{t})$, $\sigma^{2} = \sigma_{\theta}^{2}(x, \alpha_{t})$, $a = a_{t+1}$, $b = wx$, and $c^{2} = b_{t+1}$ to compute \cref{eq:noise-kernel:model}, giving us the transition model
\begin{align}
  p^{t}_{\theta}(y \mid x) = \mathcal{N}( y \mid m_{t+1}, s_{t+1}^{2} I),
\end{align}
where
\begin{align}
  m_{t+1}
  & = w x + (\alpha_{t+1} - w \alpha_{t}) \mu_{\theta}(x, \alpha_{t})
    \label{eq:nk:continuous:mean}
  \\
  s_{t+1}^{2}
  & = \beta_{t+1} - w^{2}\beta_{t} + (\alpha_{t+1} - w \alpha_{t})^{2} \sigma^{2}_{\theta}(x, \alpha_{t})
    \label{eq:nk:continuous:variance}
    .
\end{align}

\subsection{Categorical noise kernels}%
\label{sec:nk:categorical}

For categorical data, we follow~\cite{austin2021structured} and use an absorbing state noise process.
We assume the data is $D$-dimensional and that each dimension can take on one of $K$ values.
Then, the noisy data has $K+1$ categories, where the last category is an absorbing state, indicating that the underlying value has been masked.
For every element $i$ of the data, the noise distribution replaces its value $z_{i}$ by the absorbing state $K+1$ with probability $\beta_{t}$:
\begin{equation}
  p^{t}(x_{i} \mid z)
  =
    \text{Cat}\left(
    x_{i} \mid (1 - \beta_{t}) \mathbf{1}_{z_{i}} + \beta_{t} \mathbf{1}_{K+1}
    \right),
    \label{eq:nk:categorical:noise}
\end{equation}
where $\mathbf{1}_{a}$ denotes the one-hot encoding of a value $a$ in $K+1$ categories.
As in \cref{sec:nk:continuous}, we let the conditional transition distribution depend on the previous state to improve the stability of the chain:
\begin{align}
  p^{t}(y_{i} \mid z, x)
  & =
    \text{Cat}\big(
    y_{i} \mid
    (1 - b_{t+1}) w \mathbf{1}_{x_{i}} +
    \nonumber \\
  & \hspace{-0.8cm}
    + (1 - b_{t+1}) (1 - w) \mathbf{1}_{z_{i}} + b_{t+1} \mathbf{1}_{K+1}
    \big),
    \label{eq:nk:categorical:conditional}
\end{align}
where $w$ is a scalar hyperparameter controlling the mixing speed of the kernel.
We select $b_{t+1}$ to ensure the discrete analogs of conditions~\eqref{eq:noise-kernel:db} and~\eqref{eq:noise-kernel:non-equilibrium-condition}.
For the non-equilibrium case, note that \cref{eq:nk:categorical:noise,eq:nk:categorical:conditional} mean that $y_{i}$ either takes on the value $z_{i}$ or $K+1$, and that the latter happens with probability
\begin{align}
  p^{t}(y_{i} = K + 1 \mid z) = (1 - b_{t+1})w \beta_{t} + b_{t+1}.
\end{align}
We therefore must have
\begin{align}
  (1 - b_{t+1})w \beta_{t} + b_{t+1} = \beta_{t+1}
  & \nonumber \\
  & \hspace{-3.4cm} \Rightarrow
    b_{t+1} = (\beta_{t+1} - w \beta_{t}) / (1 - w \beta_{t}).
    \label{eq:nk:categorical:beta}
\end{align}
Similar to the noise kernel described in \cref{sec:nk:continuous}, we do not need to treat the equilibrium case separately:
\begin{restatable}{prop}{categoricaldb}%
  \label{prop:categorical-db}
  If the noise level is constant over time, so that $\beta_{t} = \beta$ for all $t$, then the categorical noise kernel defined by \cref*{eq:nk:categorical:noise,eq:nk:categorical:conditional,eq:nk:categorical:beta} satisfies the detailed balance criterion~\eqref{eq:noise-kernel:db}.
\end{restatable}

A proof of Proposition~\ref{prop:categorical-db} is provided in \cref{app:proofs}.

We model the denoising distribution as
\begin{align}
  \label{eq:nk:categorical:denoising}
  r^{t}_{\theta}(z_{i} \mid x) = \text{Cat}\left(z_{i} \mid f^{i}_{\theta}(x, \alpha_{t})\right)
  ,
\end{align}
where $f^{i}_{\theta}$ is a deep neural network parameterized by weights $\theta$ with output indexed by $i$.
Inserting \cref{eq:nk:categorical:conditional,eq:nk:categorical:denoising} into the discrete version of \cref{eq:noise-kernel:model}, we get the transition model
\begin{align}
  p^{t}_{\theta}(y_{i} \mid x)
  & =
    \text{Cat}\big(
    y_{i} \mid
    (1 - b_{t+1}) w \mathbf{1}_{x_{i}} +
    \nonumber \\
  & \hspace{-1.2cm}
    + (1 - b_{t+1}) (1 - w) f^{i}_{\theta}(x, \alpha_{t}) + b_{t+1} \mathbf{1}_{K+1}
    \big)
    .
\end{align}

\section{Experiments}%
\label{sec:experiments}

\begin{figure*}[t]
  \centering
  \includegraphics[width=\linewidth]{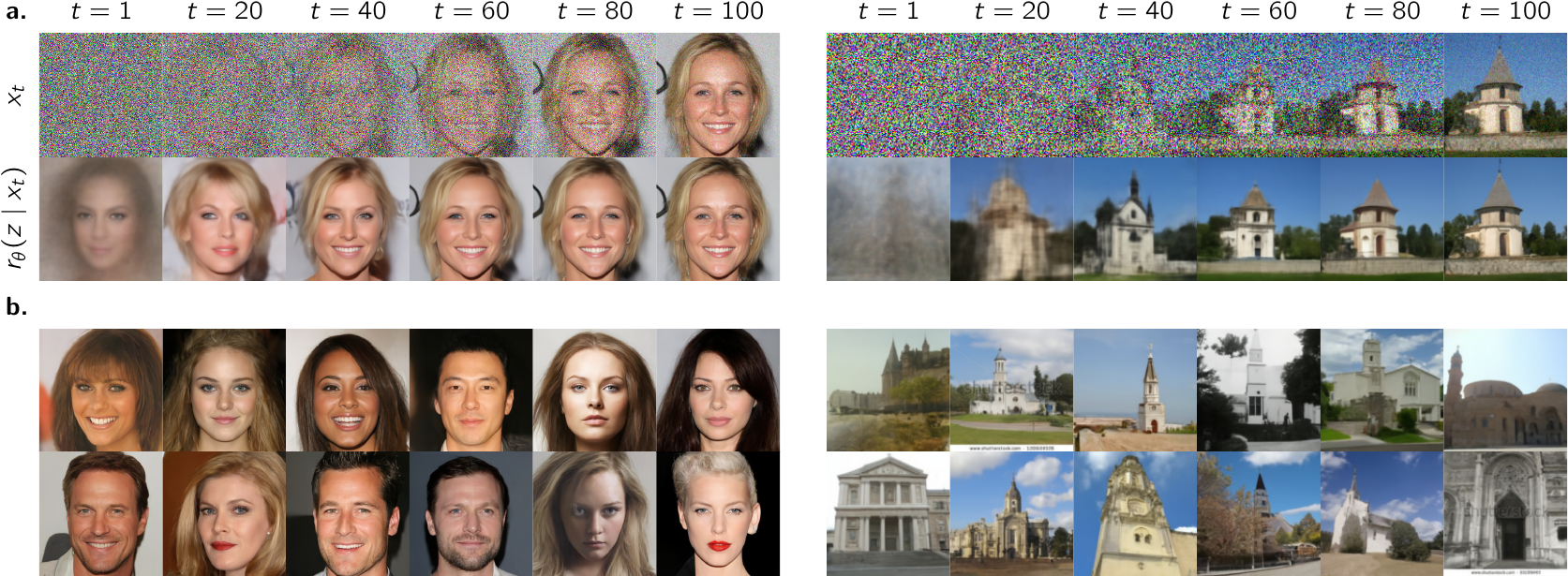}%
  \caption{%
    Data synthesis from isotropic noise.
    \textbf{a.}~Examples of the generative process.
    The noise level is annealed from $\beta_{0}=1.0$ to $\beta_{T}=0.01$ over $T=100$ steps.
    Top row: Noisy state $x_{t}$ at step $t$.
    Bottom row: Denoised state $\mathbb{E}[r_{\theta}(z \mid x_{t})]$.
    \textbf{b.}~Representative samples.
    Left: CelebA-HQ ($256 \times 256$). Right: LSUN Church ($128 \times 128$).
  }%
  \label{fig:imagesynthesis}
\end{figure*}

We evaluate the performance of noise kernels optimized by contrastive adjustment on a number of image synthesis tasks.
The transition models are defined as in \cref{sec:nk:continuous,sec:nk:categorical}, and we parameterize $r_{\theta}$ by a U-Net~\cite{ronneberger2015u} architecture similar to the one used by~\textcite{hoogeboom2021argmax}.
For continuous noise kernels, we use $w=0.5$ and set $\alpha_{t} = 1 - \beta_{t}$.
For categorical noise kernels, we use $w=0.95$.
The models are trained as in \cref{alg:ca:gradient} but with minibatch gradient descent using the Adam optimizer~\cite{kingma2014adam}.
Implementation details are provided in \cref{app:implementation}.

\subsection{Image synthesis: continuous data}

\begin{table}[t]
  \begin{center}
    \begin{tabular}{lc}
      \toprule
      Method & FID \\
      \midrule
      NCSN++~\cite{song2020score} & $2.20$ \\
      StyleGAN2-ADA~\cite{karras2020training} & $2.92$ \\
      DDPM~\cite{ho2020denoising} & $3.17$ \\
      NCSNv2~\cite{song2020score} & $10.87$ \\
      NK-CA & $18.27$ \\
      NCSN~\cite{song2019generative} & $25.32$ \\
      DenseFlow~\cite{grcic2021densely} & $34.90$ \\
      EBM~\cite{du2019implicit} & $38.2$ \\
      PixelIQN~\cite{ostrovski2018autoregressive} & $49.46$ \\
      \bottomrule
    \end{tabular}
  \end{center}
  \caption{%
    Fréchet Inception Distance~\cite{heusel2017gans} on CIFAR-10.
    Reported result for noise kernel trained with contrastive adjustment (NK-CA) is the average of 5 independent runs with a generated dataset size of \num{50000}.
    Reported results for other models are taken from the original papers.
  }%
  \label{tab:results:cifar}
\end{table}

To generate images, we sample $x_{0}\sim\mathcal{N}(0,I)$ and run the Markov chain forward while linearly annealing the noise level from $\beta_{0}=1.0$ to $\beta_{T}=0.01$ over $T=100$ steps.
The generated images are denoised by taking the mean of $r_{\theta}(z \mid x_{T})$.

We compare the image synthesis performance of noise kernels trained with contrastive adjustment to other state-of-the-art generative models on CIFAR-10 (\cref{tab:results:cifar}).
Contrastive adjustment attains an average FID score of $18.27$ over five independent runs, which is lower than the top performing normalizing flow, autoregressive, and energy-based models, but higher than the best score-based, diffusion, and adversarial models.
Uncurated samples are presented in \cref{app:imagesynthesis:cifar10}.

To investigate the performance of contrastive adjustment on higher-resolution images, we additionally train noise kernels on the LSUN Church (cropped and resized to $128 \times 128$) and CelebA-HQ (resized to $256 \times 256$) datasets.
\Cref{fig:imagesynthesis} shows traces of the sampling process and representative samples from our models.
Overall, we find generated images to be diverse and well-composed.
On LSUN Church, our model attains a competitive FID score of $9.40$.
Uncurated samples are presented in \cref{app:imagesynthesis:celebahq,app:imagesynthesis:lsunchurch}.

\subsection{Image synthesis: categorical data}

Categorical data is non-ordinal and therefore typically more challenging to model than continuous or ordered discrete data, as the relationships between categories is not self-evident.
To keep consistency with our evaluations on continuous data, we again evaluate the performance of contrastive adjustment on CIFAR-10 but note that categorical data models are not ideal for image datasets, since pixel intensity values are ordered.
While it would be possible to adapt noise kernel models for ordered discrete data, for example by restricting the transition distribution to only allow transitions between adjacent categories, we leave these adaptations for future work.

To make it easier for the model to learn how categories are related, we discretize pixel intensities into $10$ categories.
Samples are generated by setting all elements of $x_{0}$ to $\mathbf{1}_{K+1}$ and running the Markov chain for $T=500$ steps while linearly annealing $\beta_{t}$ from $1.0$ to $0.5$.
Comparing generated samples to the discretized dataset, our model attains a FID score of $14.76$.
Uncurated samples are presented in \cref{app:imagesynthesis:cifar10:categorical}.

\subsection{Variant generation}

\begin{figure*}[t]
  \centering
  \includegraphics[width=\linewidth]{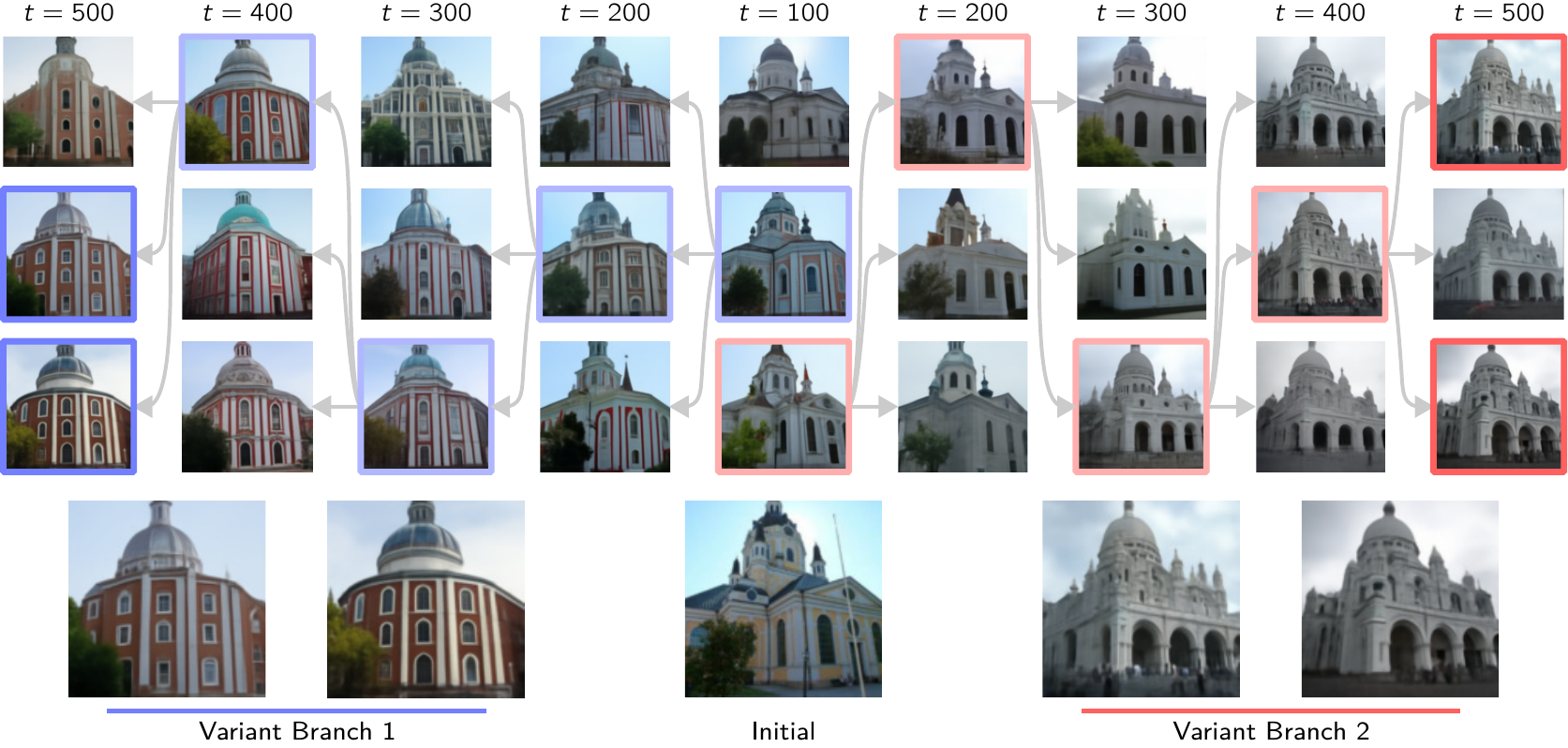}%
  \caption{%
    Variant generation on LSUN Church ($128 \times 128$).
    Candidate outputs are generated by running the chain forward $100$ steps, beginning with an example from the dataset.
    One of the candidates is selected to generate the next set of candidates, and the process is repeated until the desired output is found.
    Two different sample trajectories are shown that select for different visual attributes, illustrating how the local data manifold can be traversed by iterative refinement in order to create diverse outputs.
  }%
  \label{fig:variants}
\end{figure*}

Variant generation allows human input to guide the sampling process by iteratively refining synthesized outputs.
Such design processes could be highly valuable in many different fields, including not only the production of art and music but also, for example, the discovery of new molecules for drug development.
As an MCMC-based generative model, contrastive adjustment is especially suited for this task, as it allows for local exploration of the data manifold by sequential sampling.
Variants are generated by starting from an example $x_{0}$ and running the Markov chain forward a desired number of steps.
The noise scale $\beta_{t}$ can be adjusted to control the amount of variation added in each step.

We exemplify variant generation using contrastive adjustment in \cref{fig:variants}, where we start with an image from the LSUN Church dataset and run the Markov chain for $T=100$ steps at a constant noise level $\beta = 0.2$ in order to generate variants of the original image.
After each generation, we select one of the outputs to create new samples from and repeat the process.
To illustrate the diversity of variants that can be produced in this way, we generate two distinct trajectories that select for different visual attributes in the images.
Additional variant generation traces are provided in \cref{app:imagevariants}.

\subsection{Inpainting}

\begin{figure}[t]
  \centering
  \includegraphics[width=\linewidth]{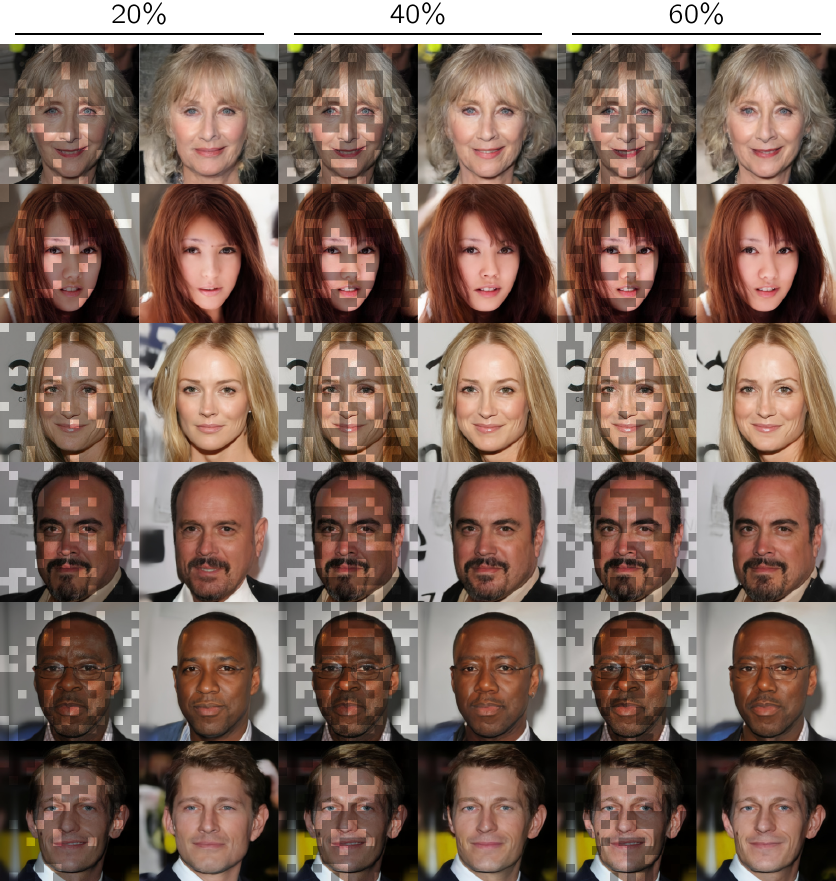}%
  \caption{%
    Inpainting experiments on CelebA-HQ ($256 \times 256$) validation set.
    Each column pair:
    To the left, the original image.
    Darkened regions are masked out and the model is conditioned on the remaining pixels.
    To the right, inpainting results after annealing the noise level from $\beta_{0}=1.0$ to $\beta_{T}=0.01$ in $T=100$ steps.
    The number above the column pair indicates the proportion of unmasked pixels in the input image.
  }%
  \label{fig:inpainting}
\end{figure}

Inpainting is the task of filling in missing regions of an image and is an important tool in image editing and restoration.
Noise kernels can be conditioned on existing data without retraining by fixing the denoising distribution $r_{\theta}$ to be a point mass at the observed image pixels.
Concretely, let $\bar{z}$ be the original image and $M$ a binary mask, where $M_{ij} = 1$ if pixel $(i, j)$ is to be inpainted and $M_{ij} = 0$ otherwise.
Where $M_{ij} = 0$, we modify the denoising model $r_{\theta}$ to be a delta distribution centered at $\bar{z}$.
In the case of continuous data, we then have instead of \cref{eq:nk:continuous:mean,eq:nk:continuous:variance}:
\begin{align}
  m_{t+1}
  & = w x + (\alpha_{t+1} - \alpha_{t}w) (
    \nonumber \\
  & \qquad
    M \mu_{\theta}(x, \alpha_{t}) + (1 - M) \bar{z} )
  \\
  s_{t+1}^{2}
  & = (\beta_{t+1} - w^{2}\beta_{t}) +
    \nonumber \\
  & \qquad
    (\alpha_{t+1} - \alpha_{t}w)^{2} M \sigma^{2}_{\theta}(x, \alpha_{t})
    .
\end{align}

\Cref{fig:inpainting} shows inpainting results on CelebA-HQ where we have masked out randomly selected $16 \times 16$ tiles of the original images.
We find that the model is able to generate coherent and realistic inpaintings even when large parts of the original image has been masked out.
Additional inpainting results for different mask types are provided in \cref{app:inpainting}.

\section{Related work}

\paragraph{Walkback and variational walkback}

Contrastive adjustment is closely related to the walkback algorithm in Generative Stochastic Networks (GSNs)~\cite{alain15_gsns}.
GSNs form a Markov chain by alternating sampling from a corruption process and a denoising distribution.
The walkback algorithm runs the chain forward a variable number of steps starting with an example from the dataset and then updates the denoising distribution to increase the backward log-likelihood of the original example.
In contrast to contrastive adjustment, GSNs use different forward and backward models and do not contrast the transition probabilities of time-adjacent states.
Instead, GSNs are trained so that the denoising model can undo, or ``walk back'' from, the noise generated by the forward process in one step.

The variational walkback algorithm~\cite{goyal2017variational} introduces a finite Markov chain that transitions from low- to high-temperature sampling.
Similar to contrastive adjustment, variational walkback uses the same forward and backward models.
However, variational walkback does not attempt to learn an equilibrium process, where the transition operator obeys detailed balance.
Instead, the algorithm is derived from a variational objective that maximizes a lower bound on the evidence of the initial state of the chain.

\paragraph{Contrastive divergence}

The idea of contrasting examples from the data distribution with one-step MCMC perturbations induced by the model being learned is similar in spirit to contrastive divergence~\cite{hinton2002training}, which is used to estimate the gradient of the log-likelihood in energy-based models.
While the learned energy model can be used for MCMC-based sampling, contrastive divergence does not explicitly learn a transition kernel.

\paragraph{Diffusion models}

Sampling data by noise annealing has been studied extensively in recent works.
Denoising diffusion probabilistic models and their derivatives~\cite{sohl2015deep,ho2020denoising} have gained wide-spread popularity and achieved state-of-the-art results on many data synthesis tasks.
Notably, the continuous and categorical noise kernels studied in this work are similar to the non-Markovian inference models proposed by~\textcite{song2021denoising} but can also be used for stationary chains.
Additionally, score-based generative models, such as noise conditional score networks (NCSNs)~\cite{song2019generative}, can, similar to contrastive adjustment, learn stationary sampling distributions.
However, they are limited to Langevin sampling and do not allow arbitrary transition operators.
As a consequence, NCSNs cannot, for example, be used to model categorical data without extensive modifications~\cite{meng2022concrete}.

\section{Limitations and future work}

We have presented a first exposition of contrastive adjustment and its applications to generative modeling.
As such, there are several promising avenues for future research.

First, we have left a formal treatment of the convergence properties of contrastive adjustment to future work.
While we have empirically found contrastive adjustment to be well-behaved, it is not clear under which conditions it is guaranteed to converge and what the convergence rate is.

Second, we have studied the performance of contrastive adjustment only on image datasets.
Nevertheless, the flexibility of contrastive adjustment and its ability to model both continuous and discrete data domains make it highly promising also for other modalities, including text and audio.

Finally, similarities between noise kernels and denoising diffusion probabilistic models open up for several potential cross-breeding opportunities.
For example, recent progress on conditional diffusion models~\cite{wang22ddnm,zhang23controlnet} could be incorporated into our models to improve their generative capabilities.
We hope future work will explore these and other research directions.

\section{Conclusion}

We have proposed contrastive adjustment for learning Markov chains whose stationary distribution matches the data distribution.
Contrastive adjustment is easy to implement and can be applied to both continuous and discrete data domains.
A notable strength of our models is their ability to efficiently generate variants of provided data, making it possible to locally explore the data manifold and to bring in human competence and guidance in iterative human-computer design loops.
We have found their performance to be close to current state-of-the-art generative models for image synthesis and expect results could be improved by future adaptations and tuning.

\section*{Broader impact}

Our research adds to a growing body of work on neural generative models.
Powerful generative models are likely to have a large impact on society in the near future.
On the one hand, they have the potential to improve productivity and creative output in a wide range of professions.
On the other hand, they may also cause considerable harm in a number of ways.
In the short term, productivity gains from generative models may cause significant worker displacement.
In the longer term, several ethical issues may arise.
For example, given their ability to generate realistic representations of real-world objects and persons, generative models may be used to create misleading content that can be used to deceive the public or for other malicious purposes.
Furthermore, when used in decision-making systems, it is important to consider potential biases in the training data, as these may be propagated by the model and thus affect the fairness of the system.
Data verifiability and model explainability will likely be key topics in mitigating these downsides.
Overall, given the disruptive potential of generative models, it is crucial that they are deployed with care and that the research community continues to engage in an open dialogue with the public about their capabilities and limitations in order to minimize potential risks while enabling the large potential benefits they stand to give to society.

\section*{Software and data availability}

Code for the experiments in this paper is available at \url{https://github.com/ludvb/nkca}.
The CIFAR-10~\cite{krizhevsky2009learning}, CelebA-HQ~\cite{karras17progressive}, and LSUN Church~\cite{yu15lsun} datasets are available from their respective official sources.

\section*{Acknowledgments}

This project has received funding from the European Research Council (ERC) under the European Union’s Horizon 2020 research and innovation programme (grant agreement no. 101021019). This work was also supported by the Knut and Alice Wallenberg foundation, the Erling-Persson family foundation, the Swedish Cancer Society, and the Swedish Research Council.

{\printbibliography}

\appendix
\onecolumn

\section{Experimental details}%
\label{app:implementation}

The denoising model $r_{\theta}$ largely follows the architecture used in~\textcite{hoogeboom2021argmax}, which is a U-Net~\cite{ronneberger2015u} with two residual blocks and a residual linear self-attention layer at each resolution.
We make some minor modifications by adding ReZero~\cite{bachlechner21rezero} to every residual connection, removing the dropout layers, and moving the self-attention layers to between the residual blocks.
Following~\cite{ho2020denoising}, the noise level is encoded using a sinusoidal position embedding~\cite{vaswani17attention} that is added to the data volume in each residual block.

Our CIFAR-10 and LSUN Church models use a channel size of $128$ and a 4-level deep U-Net with channel multipliers $(1, 2, 4, 8)$ and consist of $118$ million parameters.
Our CelebA-HQ model uses a channel size of $256$ and a 6-level deep U-Net with channel multipliers $(1, 1, 2, 2, 4, 4)$ and consists of $272$ million parameters.

Hyperparameter selection for $w$ and the final noise level $\beta_{T}$ were performed on CIFAR-10 with line search using FID as the evaluation metric.
Selected values were reused for the other datasets.

The CelebA-HQ model was trained for $\num{437500}$ iterations ($500$ epochs) with a batch size of $32$ on four NVIDIA A100 GPUs with no data augmentation.
The LSUN Church model was trained for $\num{1972500}$ iterations ($500$ epochs) with a batch size of $32$ on a single NVIDIA A100 GPU with no data augmentation.
The continuous CIFAR-10 model was trained for $\num{273700}$ iterations ($700$ epochs) with a batch size of $128$ on a single NVIDIA A100 GPU with random horizontal flips.
The categorical CIFAR-10 model was trained for $\num{49000}$ iterations ($500$ epochs) with a batch size of $512$ on a single NVIDIA A100 GPU with random horizontal flips.
We used the Adam optimizer~\cite{kingma2014adam} with a learning rate of $\num{1e-4}$ in all experiments except for the CelebA-HQ model, where we used a learning rate of $\num{2e-5}$.
Evaluation weights were computed as exponential moving averages of training weights with a decay rate of $\num{0.999}$.

FID and Inception scores were computed using Torchmetrics~\cite{torchmetrics}.
All evaluations were made against the training set of the respective dataset.
The generated dataset sizes were the same as the training set sizes, i.e., $\num{50000}$ for CIFAR-10, $\num{126227}$ for LSUN Church, and $\num{28000}$ for CelebA-HQ.

\section{Proofs}%
\label{app:proofs}

\gaussianmarginal*

\begin{proof}
  The density of the joint distribution can be written as
  \begin{align}
    p(X, Y)
    & \propto
      \exp\left(
      -\frac{(X - \mu)^{2}}{2 \sigma^{2}} - \frac{(Y - (aX + b))^{2}}{2 c^{2}}
      \right)
      \\
    & \propto
      \exp\left(
      - \left(\frac{1}{2\sigma^{2}} + \frac{a^{2}}{2c^{2}}\right) X^{2}
      - \frac{1}{2c^{2}} Y^{2}
      + \frac{a}{c^{2}} XY
      + \left(\frac{\mu}{\sigma^{2}} - \frac{ab}{c^{2}}\right) X
      + \frac{b}{c^{2}} Y
      \right)
      \label{eq:gaussianmarginal:joint}
      .
  \end{align}
  We therefore have
  \allowdisplaybreaks%
  \begin{align}
    p(Y)
    & = \int p(X, Y) \; \mathrm{d}X
      \\
    & \propto \int
      \exp\left(
        - \left(\frac{1}{2\sigma^{2}} + \frac{a^{2}}{2c^{2}}\right) X^{2}
        - \frac{1}{2c^{2}} Y^{2}
        + \left(\frac{\mu}{\sigma^{2}} + \frac{a(Y - b)}{c^{2}}\right) X
        + \frac{b}{c^{2}} Y
      \right)
      \; \mathrm{d}X
    \\
    & =
      \exp\left(
        - \frac{1}{2c^{2}} Y^{2} + \frac{b}{c^{2}} Y
      \right)
      \int
      \exp\Biggl(
      - \left(\frac{1}{2\sigma^{2}} + \frac{a^{2}}{2c^{2}}\right)
      \nonumber \\
    & \hspace{2.5cm}
      \left(
        X^{2}
        - 2 \left(\frac{\mu}{2\sigma^{2}} + \frac{a(Y - b)}{2c^{2}}\right)
        \bigg/ \left(\frac{1}{2\sigma^{2}} + \frac{a^{2}}{2c^{2}}\right)
        X
      \right)
      \Biggr)
      \; \mathrm{d}X
    \\
    & =
      \exp\left(
        - \frac{1}{2c^{2}} Y^{2} + \frac{b}{c^{2}} Y
      \right)
      \exp\left(
        \left(\frac{\mu}{2\sigma^{2}} + \frac{a(Y - b)}{2c^{2}}\right)^{2}
        \bigg/ \left(\frac{1}{2\sigma^{2}} + \frac{a^{2}}{2c^{2}}\right)
      \right)
      \nonumber \\
    & \hspace{2.5cm}
      \int
      \exp\left(
        -\left(\frac{1}{2\sigma^{2}} + \frac{a^{2}}{2c^{2}} \right)
        \left(
          X
          - \left(\frac{\mu}{2\sigma^{2}} + \frac{a(Y - b)}{2c^{2}}\right)
          \bigg/ \left(\frac{1}{2\sigma^{2}} + \frac{a^{2}}{2c^{2}}\right)
        \right)^{2}
      \right)
      \; \mathrm{d}X
    \\
    & \propto
      \exp\left(
        - \frac{1}{2c^{2}} Y^{2}
        + \frac{b}{c^{2}} Y
        + \left(\frac{\mu}{2\sigma^{2}} + \frac{a(Y - b)}{2c^{2}}\right)^{2}
        \bigg/ \left(\frac{1}{2\sigma^{2}} + \frac{a^{2}}{2c^{2}}\right)
      \right)
      \label{eq:proof:gaussianmarginal:summedintegral}
    \\
    & \propto
      \exp\left(
        - \left(
          \frac{1}{2c^{2}}
          - \frac{a^{2}}{4c^{4}} \bigg/ \left(\frac{1}{2\sigma^{2}} + \frac{a^{2}}{2c^{2}}\right)
        \right) Y^{2}
        + \left(
          \frac{b}{c^{2}}
          + \frac{a}{c^{2}} \left(\frac{\mu}{\sigma^{2}} - \frac{ab}{c^{2}}\right)
          \bigg/ \left(\frac{1}{\sigma^{2}} + \frac{a^{2}}{c^{2}}\right)
        \right) Y
      \right)
    \\
    & =
      \exp\left(
        - \frac{1}{2(c^{2} + a^{2} \sigma^{2})} Y^{2}
        + 2 \frac{a \mu + b}{2(c^{2} + a^{2} \sigma^{2})} Y
      \right) \\
    & \propto
      \exp\left(
        - \frac{\left(Y - (a \mu + b)\right)^{2}}{2(c^{2} + a^{2} \sigma^{2})}
      \right)
      \label{eq:proof:gaussianmarginal:final}
      ,
  \end{align}
  where the proportionality~\eqref{eq:proof:gaussianmarginal:summedintegral} follows from the fact that the integrand is proportional to a Gaussian density and the integral of a density over its support is equal to one.
  \Cref{eq:proof:gaussianmarginal:final} means $p(Y) = C N(Y \mid a\mu + b, c^{2} + a^{2}\sigma^{2})$ for some constant $C$.
  Since $p(Y)$ is a density function, we must have $C = 1$, completing the proof.
\end{proof}

\gaussiandb*

\begin{proof}
  Consider an element $i$ of the data.
  From \cref{eq:nk:continuous:conditional,eq:nk:continuous:alpha,eq:nk:continuous:beta}, we have
  \begin{align}
    p(y_{i} \mid z, x_{i}) = N(y_{i} \mid wx_{i} + (1 - w)\alpha z_{i}, (1 - w^{2}) \beta).
    \label{eq:proof:gaussiandb:conditional}
  \end{align}
  Using~\eqref{eq:gaussianmarginal:joint} with $X = x_{i}$, $Y = y_{i}$, $\mu = \alpha z_{i}$, $\sigma^{2} = \beta$, $a = w$, $b = (1 - w) \alpha z_{i}$, and $c^{2} = (1 - w^{2}) \beta$,
  \begin{align}
    p(x_{i} \mid z)p(y_{i} \mid z, x_{i})
    & \propto
      \exp\biggl(
      - \left(\frac{1}{2\beta} + \frac{w^{2}}{2(1 - w^{2}) \beta}\right) x_{i}^{2}
      - \frac{1}{2(1 - w^{2}) \beta} y_{i}^{2}
      + \frac{w}{(1 - w^{2}) \beta} x_{i}y_{i}
      \nonumber \\
    & \hspace{4.0cm}
      + \left(\frac{\alpha z_{i}}{\beta} - \frac{w (1 - w) \alpha z_{i}}{(1 - w^{2}) \beta}\right) x_{i}
      + \frac{(1 - w) \alpha z_{i}}{(1 - w^{2}) \beta} y_{i}
      \biggl)
      \\
    & =
      \exp\left(
        - \frac{1}{2(1 - w^{2}) \beta} \left(x_{i}^{2} + y_{i}^{2}\right)
        + \frac{w}{(1 - w^{2}) \beta} x_{i}y_{i}
        + \frac{\alpha z_{i}}{(1 + w) \beta} \left(x_{i} + y_{i}\right)
      \right)
      \label{eq:proof:gaussiandb:joint}
      .
  \end{align}
  We can see that~\eqref{eq:proof:gaussiandb:joint} is symmetric in $x_{i}$ and $y_{i}$, whereby $p(x_{i} \mid z)p(y_{i} \mid z, x_{i}) = p(y_{i} \mid z)p(x_{i} \mid z, y_{i})$ for all $x_{i},y_{i}$.
  It follows that
  \begin{align}
    p(x \mid z)p(y \mid z, x)
    =
    \prod_{i} p(x_{i} \mid z)p(y_{i} \mid z, x_{i})
    =
    \prod_{i} p(y_{i} \mid z)p(x_{i} \mid z, y_{i})
    =
    p(x \mid z)p(y \mid z, x)
    ,
  \end{align}
  as required.
\end{proof}

\categoricaldb*

\begin{proof}
  Consider an element $i$ of the data.
  Sampling $X_{t}^{i}$ and $X_{t+1}^{i}$ by \cref{eq:nk:categorical:noise,eq:nk:categorical:conditional}, there are four possible outcomes: $(X_{t}^{i}, X_{t+1}^{i}) = (z_{i}, z_{i})$, $(X_{t}^{i}, X_{t+1}^{i}) = (K + 1, K + 1)$, $(X_{t}^{i}, X_{t+1}^{i}) = (z_{i}, K + 1)$, and $(X_{t}^{i}, X_{t+1}^{i}) = (K + 1, z_{i})$.
  The first two outcomes are symmetric in $X_{t}^{i}$ and $X_{t+1}^{i}$ and therefore trivially balanced.
  It remains to show that the last two outcomes have equal probability.
  Using~\eqref{eq:nk:categorical:beta} and the fact that $\beta_{t} = \beta$ for all $t$, we have $b = \beta (1 - w) / (1 - w \beta)$ and hence
  \begin{align}
    p(X_{t}^{i} = z_{i} \mid z) p(X_{t+1}^{i} = K + 1 \mid z, X_{t}^{i} = z_{i})
    & = (1 - \beta) b
      \\
    & = \beta (1 - w) (1 - \beta) / (1 - w \beta)
      \\
    p(X_{t}^{i} = K + 1 \mid z) p(X_{t+1}^{i} = z_{i} \mid z, X_{t}^{i} = K + 1)
    & = \beta (1 - b) (1 - w)
      \\
    & = \beta (1 - w) (1 - (1 - w)\beta  / (1 - w \beta))
      \\
    & = \beta (1 - w) (1 - \beta) / (1 - w \beta)
      .
  \end{align}
  Therefore, $p(x_{i} \mid z)p(y_{i} \mid z, x_{i}) = p(y_{i} \mid z)p(x_{i} \mid z, y_{i})$ for all $x_{i},y_{i}$.
  Since $p(x \mid z)$ and $p(y \mid z, x)$ are element-wise factorized, this also means that
  \begin{align}
    p(x \mid z)p(y \mid z, x)
    =
    \prod_{i} p(x_{i} \mid z)p(y_{i} \mid z, x_{i})
    =
    \prod_{i} p(y_{i} \mid z)p(x_{i} \mid z, y_{i})
    =
    p(y \mid z)p(x \mid z, y)
    ,
  \end{align}
  as required.
\end{proof}

\newpage
\section{Extended image synthesis results: LSUN Church}%
\label{app:imagesynthesis:lsunchurch}
\begin{figure}[h!]
  \centering
  \includegraphics[width=\linewidth]{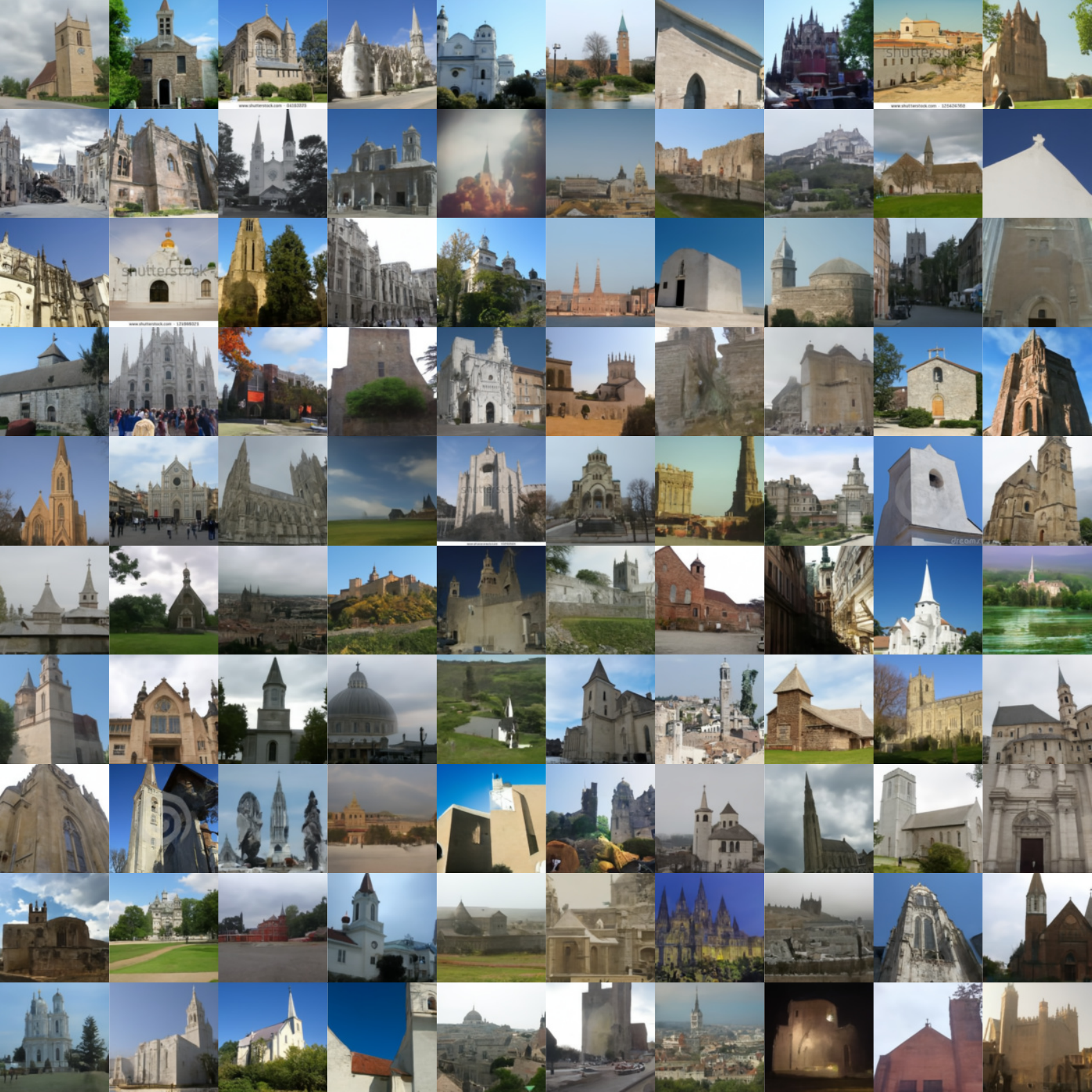}%
  \caption{
    Uncurated samples of a continuous noise kernel trained with contrastive adjustment on LSUN Church ($128 \times 128$).
    FID:~$9.40$.
  }%
\end{figure}

\newpage
\section{Extended image synthesis results: CelebA-HQ}%
\label{app:imagesynthesis:celebahq}
\begin{figure}[h!]
  \centering
  \includegraphics[width=\linewidth]{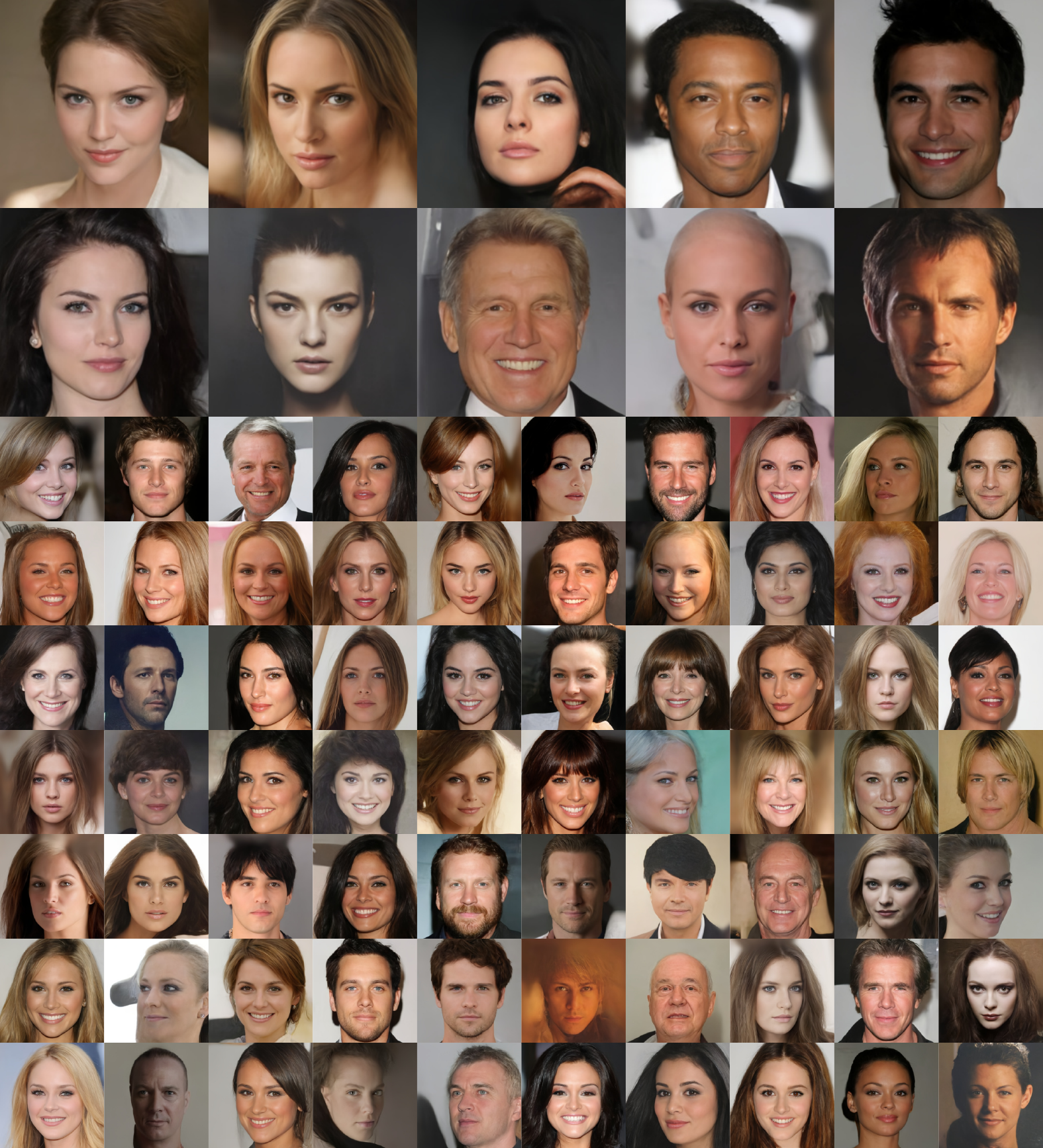}%
  \caption{
    Uncurated samples of a continuous noise kernel trained with contrastive adjustment on CelebA-HQ ($256 \times 256$).
    FID:~$31.04$.
  }%
\end{figure}

\newpage
\section{Image synthesis results: CIFAR-10}%
\label{app:imagesynthesis:cifar10}

\begin{figure}[h!]
  \centering
  \includegraphics[width=\linewidth]{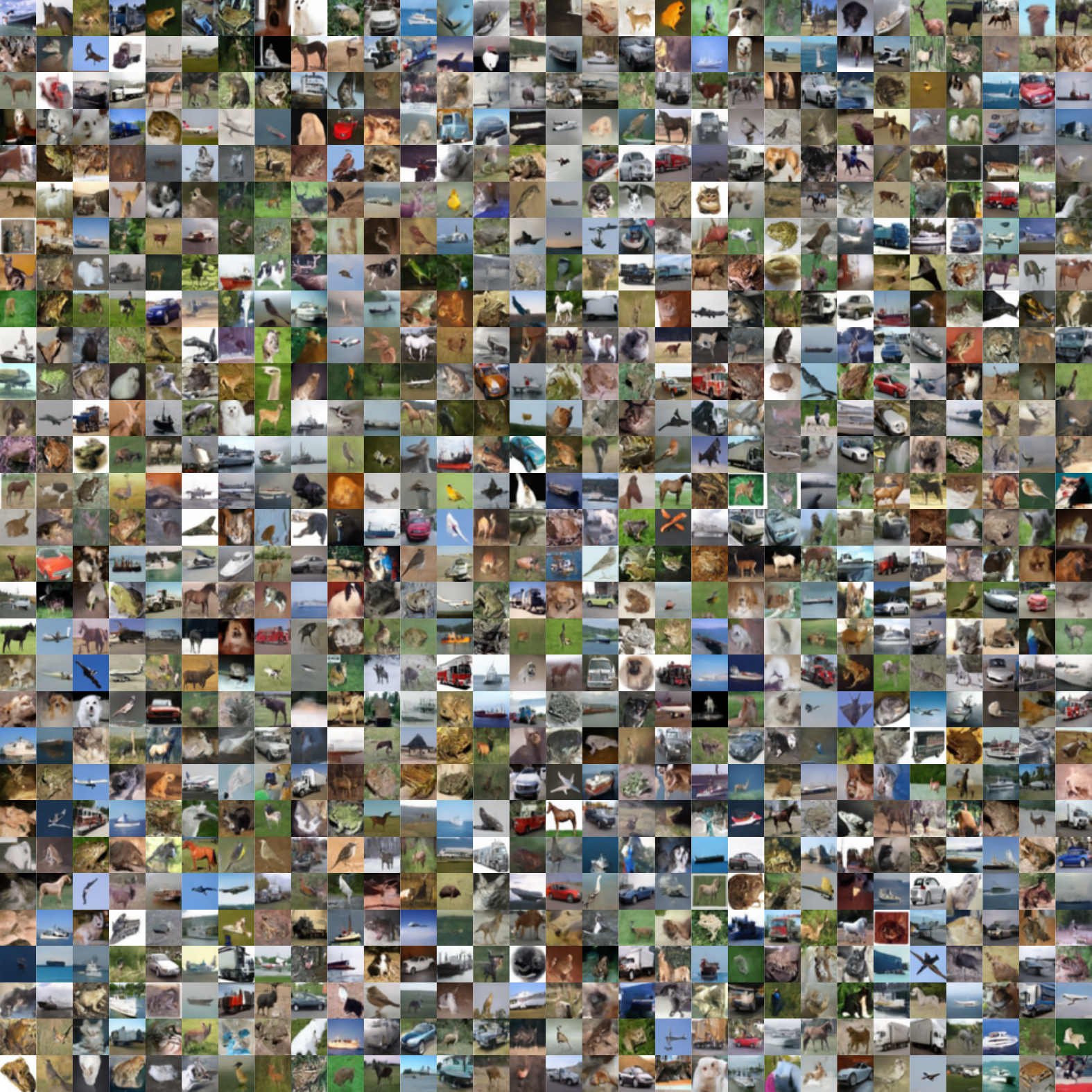}%
  \caption{
    Uncurated samples of a continuous noise kernel trained with contrastive adjustment on CIFAR-10.
    Best performing model out of five independent runs.
    FID:~$17.39$. Inception score:~$8.16 \pm 0.11$.
  }%
\end{figure}

\newpage
\section{Image synthesis results: CIFAR-10 (categorical)}%
\label{app:imagesynthesis:cifar10:categorical}

\begin{figure}[h!]
  \centering
  \includegraphics[width=\linewidth]{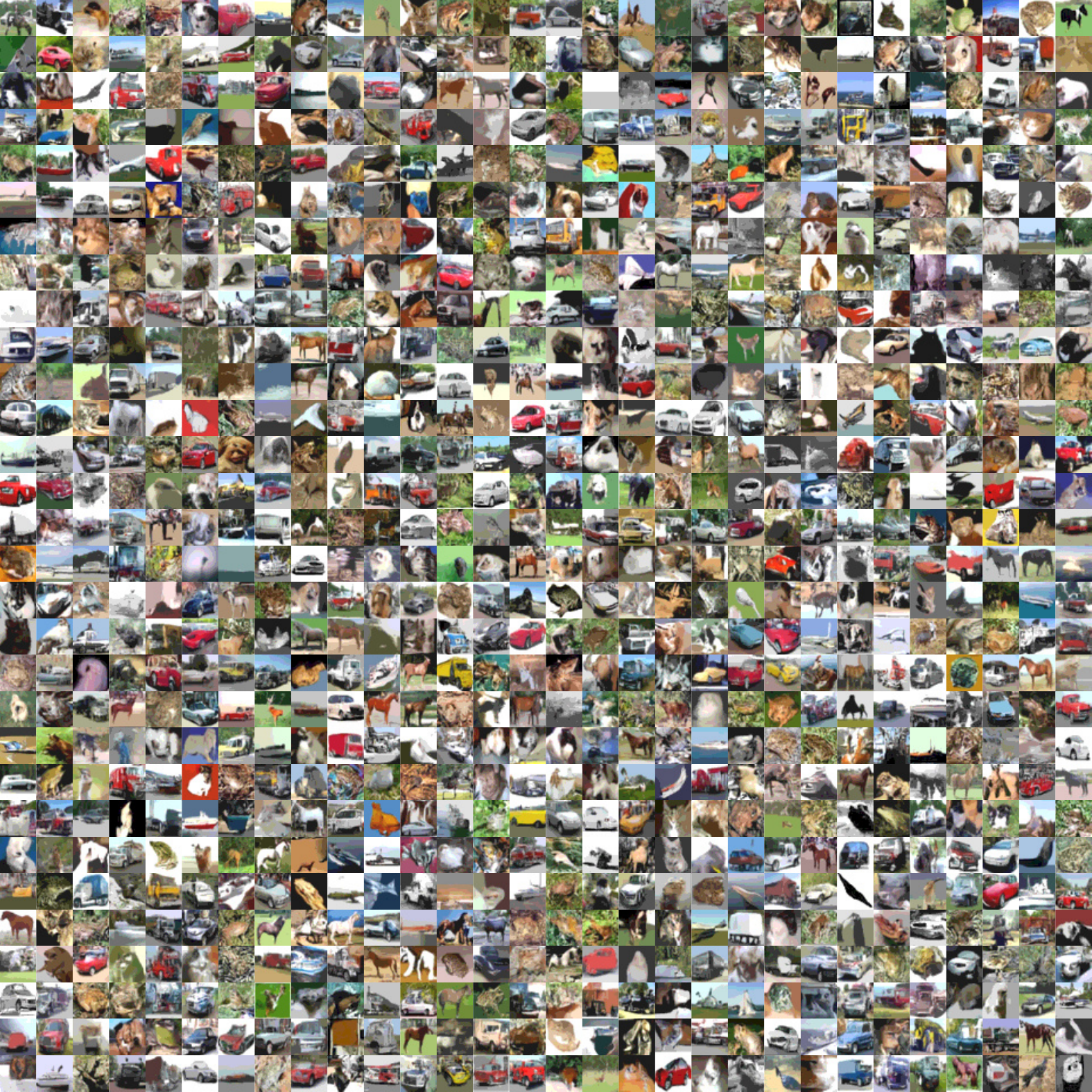}%
  \caption{
    Uncurated samples of a categorical noise kernel trained with contrastive adjustment on CIFAR-10 discretized to 10 intensity levels per color channel.
    FID:~$14.76$. Inception score:~$7.45 \pm 0.13$
  }%
\end{figure}

\newpage
\section{Extended variant generation results}%
\label{app:imagevariants}

\begin{figure}[h!]
  \centering
  \includegraphics[width=\linewidth]{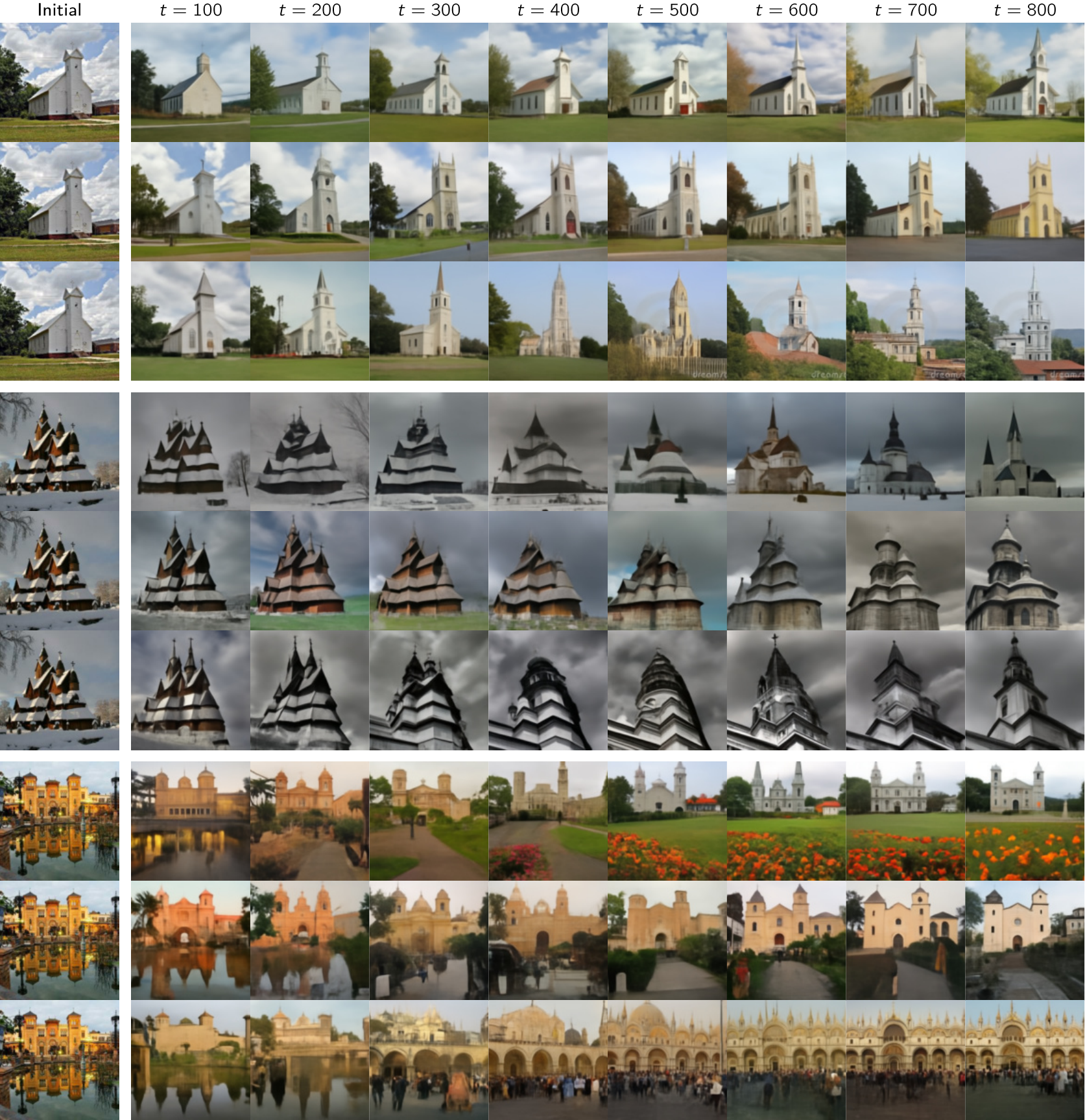}%
  \caption{
    Variant generation traces of a noise kernel trained with contrastive adjustment on LSUN Church.
    Left: Initial samples from dataset.
    Right: Progressive variants.
    Each row shows an independent run.
  }%
\end{figure}

\newpage
\section{Extended inpainting results}%
\label{app:inpainting}

\begin{figure}[h!]
  \centering
  \includegraphics[width=\linewidth]{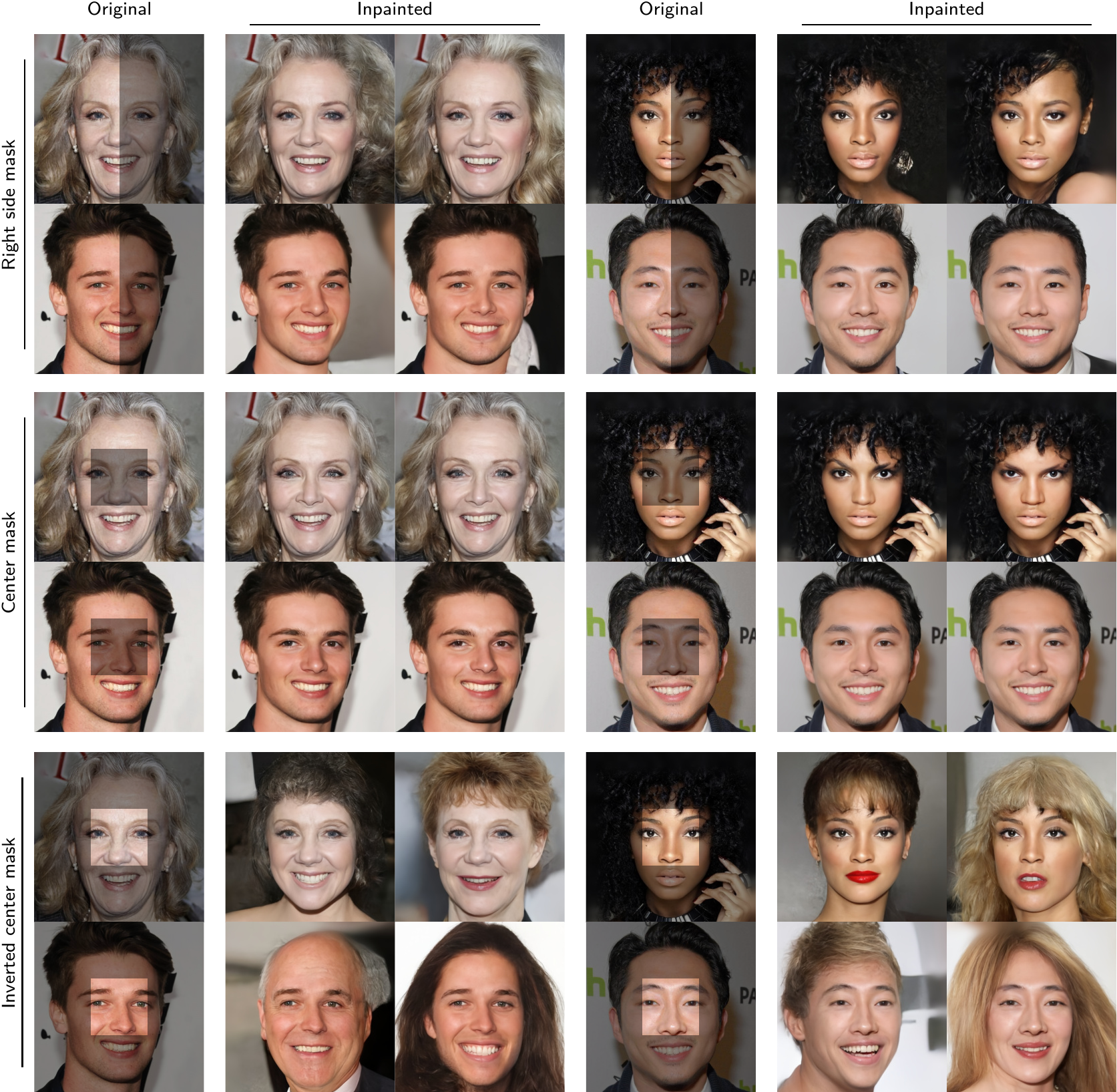}%
  \caption{
    Inpainting results on CelebA-HQ validation set.
    Darkened regions in the original images have been masked out and inpainted by a noise kernel conditioned on the remaining image.
  }%
\end{figure}

\end{document}